\documentclass{article}
\usepackage{spconf,amsmath,graphicx,hyperref}
\usepackage{amsthm}

\newtheorem{theorem}{Theorem}

\title{Multi-Modal Robust Enhancement for Coastal Water Segmentation: \\
A Systematic HSV-Guided Framework}
\name{Zhen Tian$^{1}$, Christos Anagnostopoulos$^{1,*}$, Qiyuan Wang$^{1}$, Zhiwei Gao$^{2}$\thanks{$*$ Corresponding Author}}
\address{$^{1}$ School of Computing Science, University of Glasgow\\
         $^{2}$ School of Computing Engineering, University of Glasgow}

\usepackage{amsmath,amssymb}
\usepackage{algorithmic}
\usepackage{array}
\usepackage{graphicx}
\usepackage{textcomp}
\usepackage{xcolor}
\usepackage{amsthm}
\usepackage{xcolor}
\definecolor{CustomBlue}{RGB}{25, 25, 112}      
\definecolor{CustomGreen}{RGB}{0, 100, 0}       
\definecolor{CustomPurple}{RGB}{102, 51, 153}   
\definecolor{CustomTeal}{RGB}{0, 128, 128}      

\hypersetup{
    colorlinks=true, 
    linkcolor=black, 
    citecolor=black, 
    urlcolor=CustomTeal  
}
\begin{document}
%
\maketitle
\begin{abstract}
Coastal water segmentation from satellite imagery presents unique challenges due to complex spectral characteristics and irregular boundary patterns.
Traditional RGB-based approaches often suffer from training instability and poor generalization in diverse maritime environments. 
This paper introduces a systematic robust enhancement framework, referred to as Robust U-Net, that leverages HSV color space supervision and multi-modal constraints for improved coastal water segmentation. Our approach integrates five synergistic components: 
HSV-guided color supervision, gradient-based coastline optimization, morphological post-processing, sea area cleanup, and connectivity control. Through comprehensive ablation studies, we demonstrate that HSV supervision provides the highest impact (0.85 influence score), while the complete framework achieves superior training stability (84\% variance reduction) and enhanced segmentation quality. 
Our method shows consistent improvements across multiple evaluation metrics while maintaining computational efficiency. 
For reproducibility, our training configurations and code are available here
\footnote{https://github.com/UofgCoastline/ICASSP-2026-Robust-Unet}.
\end{abstract}
\section{Introduction}
\label{sec:intro}
Coastal water segmentation from satellite and aerial imagery is a core task in remote sensing, with applications in environmental monitoring, climate studies, and coastal management~\cite{muir2024vedgesat}. Traditional RGB-based methods struggle in maritime scenes due to spectral variability, atmospheric effects, and complex shorelines~\cite{garcia2015differentiating,dierssen2021living,yan2023assessing}.
Deep learning methods, especially U-Net variants, achieve strong results in semantic segmentation~\cite{wu2024unet}. However, for coastal waters, they face several issues: 
(1) unstable training from class imbalance and spectral similarity, 
(2) weak boundary detection in complex coasts, 
(3) false positives in deep sea, and 
(4) unrealistic coastline connectivity~\cite{li2023u,zhao2023improved,attya2025hybrid,mahmoud2025bdcn_unet,wang2025semantic,gonzalez2022deep}.

Our key idea is that water bodies show clear patterns in the HSV color space~\cite{qing2025hg2former}, with lower saturation and stable hue ranges across lighting and water types. In addition, coastal boundaries follow geometric rules that can guide segmentation. Our contribution is summarized as follows:

\begin{itemize}
\item We prove that each loss component is Lipschitz continuous, and that the composite objective ensures convergence of gradient descent (Theorem~1 in Section~3.2), thereby guaranteeing reliable optimization.  
\item We show that these spectral–geometric constraints shrink the effective hypothesis space (Theorem~2 in Section~3.2), providing theoretical support for their stabilizing role in water–land separation.  
\item We provide comprehensive experiments that demonstrate clear performance gains over state-of-the-art baselines in both stability and boundary accuracy.  
\end{itemize}

\section{Related Work}
\subsection{Deep Learning for Water Segmentation}
Traditional water segmentation methods relied on spectral indices such as NDWI and MNDWI \cite{mahmood2025drought,justiniano2025new}. With the advent of deep learning, convolutional neural networks have become the dominant approach. 
U-Net \cite{liu2025mim} and its variants have shown particular success in biomedical and remote sensing applications. Recent works have explored various architectural improvements including attention mechanisms \cite{jonnala2025dsia}, multi-scale fusion \cite{gao2025msfm}, and encoder-decoder enhancements. However, most approaches ignore addressing the fundamental challenges of color space representation and training stability.

\subsection{Color Space Utilization}
While RGB remains the standard representation, alternative color spaces have demonstrated advantages in specific applications. HSV color space has been successfully applied in object tracking \cite{chen2025hyperspectral}, and agricultural monitoring \cite{johari2025corn}. However, its systematic integration into deep learning frameworks for water segmentation remains underexplored.

\subsection{Robust Training in Semantic Segmentation}
Training stability in semantic segmentation has gained attention due to gradient instability and convergence issues. Various regularization techniques including label smoothing, focal loss, and consistency training have been proposed. 
Our work differs in that we explicitly integrate spectral priors using HSV-based supervision and geometric constraints into the training objective. Unlike generic regularization, these domain-informed components reduce the effective hypothesis space and enforce spatial coherence, leading to more stable and interpretable convergence, as will be shown in our evaluation.

\section{Methodology}
We introduce a novel framework, coined Robust U-Net, for robust coastline delineation that integrates physics-informed supervision, geometric regularization, and regional consistency enforcement. 
Robust U-Net is grounded in two fundamental theoretical properties that ensure convergence guarantees and hypothesis space reduction, providing both theoretical rigor and practical effectiveness for coastal boundary extraction in satellite imagery.

\subsection{Overview}
Let $\mathbf{I} \in \mathbb{R}^{H \times W \times 3}$ be an input satellite image of dimensions $H \times W$ with three RGB channels. Our objective is to learn a mapping $f_\theta: \mathbb{R}^{H \times W \times 3} \rightarrow [0,1]^{H \times W}$ parameterized by $\theta$, which produces a water probability mask $\mathbf{M} = f_\theta(\mathbf{I})$, where $M_{ij}$ represents the probability that pixel $(i,j)$ belongs to a water region. 
Our methodology achieves robust coastline delineation through a synergistic design of multiple loss components and post-processing, formulated as a unified optimization problem that balances semantic accuracy, spectral consistency, geometric regularity, and spatial coherence.

\subsection{Theoretical Guarantees}
We provide two theoretical results that formalize the robustness of our framework. 
The first establishes the convergence properties of our composite loss, while the second demonstrates how spectral and geometric priors reduce the effective hypothesis space.
Proof sketches are provided below; full technical details of the proofs are omitted due to space limitations.

\begin{theorem}[Convergence Guarantee]
Each component loss $\mathcal{L}_i$ in our framework is Lipschitz continuous with constant $L_i$, ensuring that the composite objective satisfies
\begin{equation}
\|\mathcal{L}_{Robust}(\mathbf{M}_1) - \mathcal{L}_{Robust}(\mathbf{M}_2)\| \leq L \|\mathbf{M}_1 - \mathbf{M}_2\|
\end{equation}
where $L = \max_i \lambda_i L_i$ and $\lambda_i$ are the loss balancing coefficients. 
This guarantees convergence of gradient descent to a stationary point at rate $O(1/\sqrt{T})$, for $T$ iterations.
\end{theorem}

\begin{proof}[Sketch of Proof]
The result follows from the Lipschitz continuity of each $\mathcal{L}_i$, which ensures bounded gradient differences. The composite objective inherits Lipschitz continuity with constant $L=\max_i \lambda_i L_i$. 
By standard convergence analysis of stochastic gradient descent on Lipschitz-continuous nonconvex objectives~\cite{ghadimi2013stochastic}, 
this leads to $O(1/\sqrt{T})$ convergence to a stationary point.
Due to space limitations, we omit the full technical details.
\end{proof}

\begin{theorem}[Hypothesis Space Reduction]
The combined influence of spectral priors and geometric regularizers reduces the effective hypothesis space according to
\begin{equation}
\mathcal{H}_{\text{Reduced}} \approx \rho \cdot \mathcal{H}_{\text{Full}}
\end{equation}
where $\rho \propto \operatorname{Corr}(\text{HSV features}, \text{true water labels})$ quantifies the correlation between HSV-based indicators and ground truth water labels, and $\operatorname{Corr}(\cdot)$ denotes statistical correlation. 
\end{theorem}

\begin{proof}[Sketch of Proof]
Domain-informed priors act as additional inductive bias, constraining the hypothesis space to models consistent with spectral and geometric properties. 
Formally, introducing constraints $K$ yields a reduced space $H'=\{h \in H : h \models K\}$ with $H' \subseteq H$. 
Hence, the effective complexity scales with the strength of correlation $\rho$ between priors and ground-truth labels. 
This parallels results on hypothesis space reduction via domain knowledge~\cite{yu2007incorporating}.
We omit the full proofs because of space limitations.
\end{proof}

\subsection{Physics-Informed HSV Supervision}
We incorporate domain-specific color priors through HSV color space analysis, thus, leveraging the physical properties of water spectral signatures. For each pixel $(i,j)$, we convert the RGB values to HSV representation $(H_{ij}, S_{ij}, V_{ij})$ and compute a physics-informed water likelihood:
\begin{equation}
P_{HSV}(i,j) = \sigma\!\left(\alpha_H H_{ij} + \alpha_S S_{ij} + \alpha_V V_{ij} + \beta\right)
\end{equation}
where $\sigma(\cdot)$ is the sigmoid function, $\{\alpha_H, \alpha_S, \alpha_V, \beta\}$ are learnable coefficients that capture the statistical relationship between HSV values and water presence, and $P_{HSV}(i,j) \in [0,1]$ represents the HSV-based water probability. The HSV-guided supervision loss is then formulated as:
\begin{equation}
\mathcal{L}_{HSV} = \frac{1}{HW} \sum_{i=1}^{H}\sum_{j=1}^{W} \left|m_{ij} - P_{HSV}(i,j)\right|^2 \cdot w_{ij}
\end{equation}
where $m_{ij} = M_{ij}$ is the predicted water probability at pixel $(i,j)$, and $w_{ij}$ is an adaptive confidence weight defined as:
\begin{equation}
w_{ij} = \exp\!\left(-\frac{d_{HSV}(i,j)^2}{2\sigma^2}\right)
\end{equation}
where $d_{HSV}(i,j)$ measures the Euclidean distance from pixel $(i,j)$ to the optimal water representation, $\sigma$ is a bandwidth parameter controlling the sensitivity of the weighting function, and $w_{ij}$ assigns higher importance to pixels with strong water-like HSV signatures.

\subsection{Geometric Coastline Regularization}
To enforce geometric consistency and topological coherence in coastline predictions, we introduce two complementary regularization terms that address different aspects of boundary quality. We first extract candidate coastline pixels through morphological operations, computing the coastline pixel set:
\begin{equation}
\mathbf{C} = \text{Dilate}(\mathbf{M}, k) - \text{Erode}(\mathbf{M}, k),
\end{equation}
where $\text{Dilate}(\cdot, k)$ and $\text{Erode}(\cdot, k)$ are morphological dilation and erosion operations, $k$ is the size of the structuring element (typically $k = 3$ for a $3 \times 3$ kernel), and $\mathbf{C}$ represents the set of boundary pixels between water and land regions. To ensure geometric consistency, we penalize irregular gradient fluctuations through a smoothness regularization term:
\begin{equation}
\mathcal{L}_{Coast} = \frac{1}{|\mathbf{C}|} \sum_{(i,j)\in \mathbf{C}} \|\nabla \mathbf{M}(i,j)\|_2^2,
\end{equation}
where $|\mathbf{C}|$ is the cardinality of the coastline pixel set, $\nabla \mathbf{M}(i,j) = [\frac{\partial M}{\partial x}, \frac{\partial M}{\partial y}]_{(i,j)}$ is the spatial gradient of the mask at pixel $(i,j)$, and $\|\cdot\|_2$ denotes the $L_2$ norm. This term encourages smooth transitions, effectively reducing jagged artifacts and promoting natural coastline curvature.

Furthermore, we enforce topological consistency by penalizing columns with multiple disconnected water regions through a connectivity constraint:
\begin{equation}
\mathcal{L}_{Conn} = \sum_{x=1}^{W} \max\!\left(0, \frac{\text{ConnectedRegions}(\mathbf{M}[:,x])-1}{\text{MaxRegions}}\right)
\end{equation}
where $\mathbf{M}[:,x]$ represents the $x$-th column of the water mask, $\text{ConnectedRegions}(\cdot)$ counts the number of connected components in a 1D binary sequence, $\text{MaxRegions}$ is a normalization factor preventing unbounded loss values, and the $\max(0, \cdot)$ operation ensures non-negative penalties.

\subsection{Regional Consistency Enforcement}

To promote spatial homogeneity in large water bodies and reduce prediction noise within uniform regions, we introduce a variance-based sea cleanup regularizer:
\begin{equation}
\mathcal{L}_{\text{Sea}} = \frac{1}{|\mathbf{S}|} \sum_{(i,j)\in \mathbf{S}} \operatorname{Var}\!\big(\mathbf{M}[\mathcal{N}(i,j)]\big)
\end{equation}
where $\mathbf{S}$ denotes the set of detected sea pixels (identified through connected component analysis with area thresholding), $\mathcal{N}(i,j)$ represents a local neighborhood around pixel $(i,j)$ (typically a $5 \times 5$ window), $\mathbf{M}[\mathcal{N}(i,j)]$ extracts the mask values within the neighborhood, and $\operatorname{Var}(\cdot)$ computes the statistical variance. 

\subsection{Unified Robust Objective}

All components are integrated into a single robust objective function that balances multiple complementary aspects of coastline quality:
\begin{equation}
\begin{split}
\mathcal{L}_{Robust} 
   &= \lambda_{CE} \mathcal{L}_{CE} + \lambda_{HSV} \mathcal{L}_{HSV} \\
   &\quad + \lambda_{Coast} \mathcal{L}_{Coast} 
   + \lambda_{Conn} \mathcal{L}_{Conn} 
   + \lambda_{Sea} \mathcal{L}_{Sea}
\end{split}
\end{equation}
where $\mathcal{L}_{CE}$ is the standard pixel-wise cross-entropy loss for semantic segmentation, defined as:
\begin{equation}
\mathcal{L}_{CE} = -\frac{1}{HW} \sum_{i=1}^{H}\sum_{j=1}^{W} \left[ y_{ij} \log(m_{ij}) + (1-y_{ij}) \log(1-m_{ij}) \right]
\end{equation}
where $y_{ij} \in \{0,1\}$ is the ground truth label for pixel $(i,j)$, and $\{\lambda_{CE}, \lambda_{HSV}, \lambda_{Coast}, \lambda_{Conn}, \lambda_{Sea}\}$ are coefficients that control the relative importance of each loss component.

The unified objective strategically balances four fundamental aspects of robust coastline detection: semantic fidelity through $\mathcal{L}_{CE}$ ensures accurate pixel-level classification, spectral consistency via $\mathcal{L}_{HSV}$ incorporates physics-informed color priors, geometric regularity enforced by $\mathcal{L}_{Coast}$ and $\mathcal{L}_{Conn}$ promotes smooth and topologically consistent coastlines, and regional uniformity achieved through $\mathcal{L}_{Sea}$ maintains spatial coherence in homogeneous regions. 

\section{Experimental Evaluation}
\subsection{Datasets}
Experiments are conducted on coastal imagery from St Andrews region, Scotland, acquired via Sentinel-2 satellite images from May 2017 to May 2025. The dataset comprises 954 high-resolution images with pixel-level water segmentation annotations. Images utilize NIR-Red-Green band combination for enhanced water-land contrast. Dataset split: 80\% training, 20\% validation with stratified sampling. We compare Robust U-Net against the Traditional U-Net with identical architecture but standard cross-entropy loss without robust enhancements. This isolates the contribution of our proposed components.
\begin{figure}[t]
  \centering
  \includegraphics[width=0.9\linewidth]{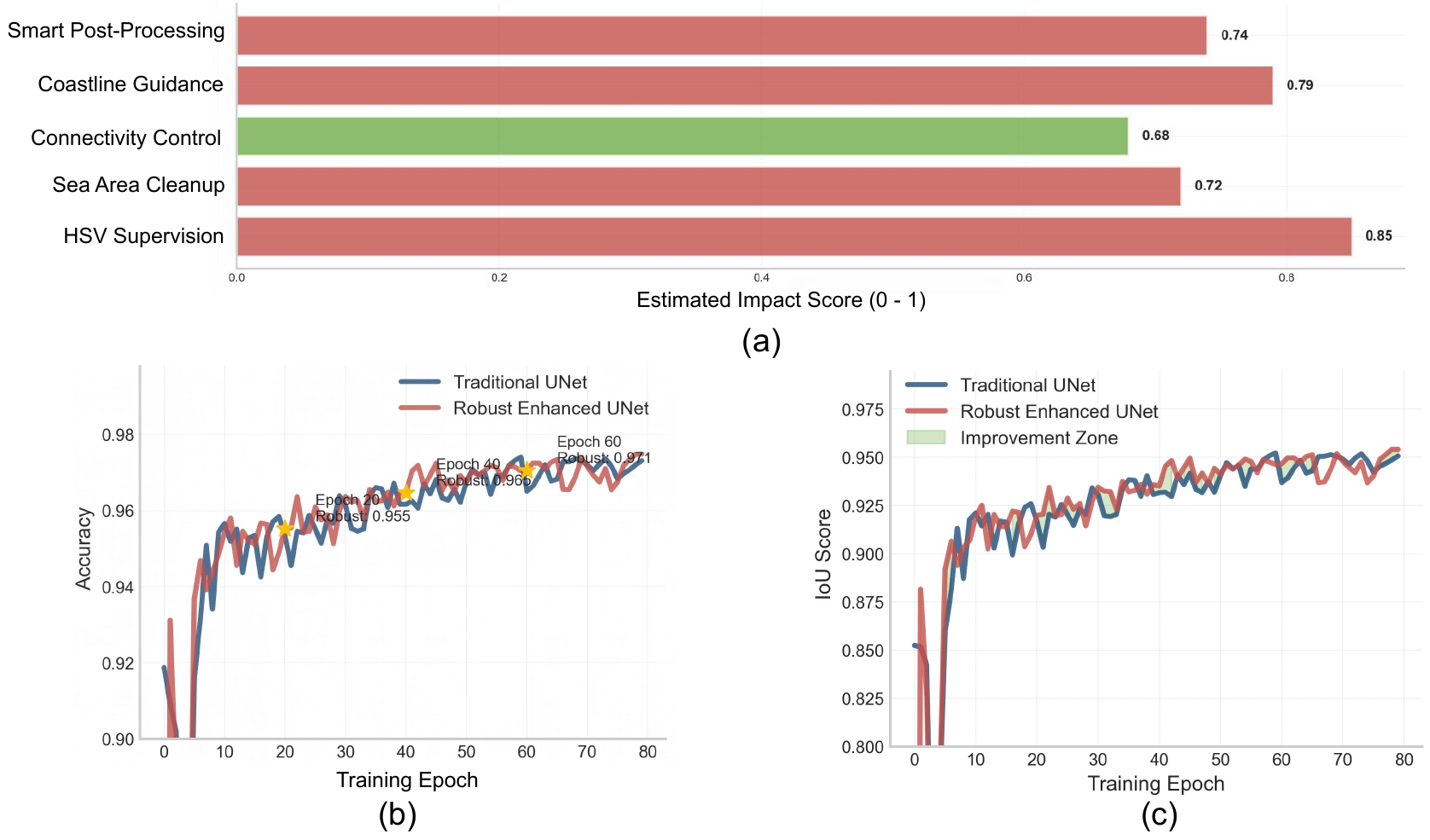}
  \caption{Comprehensive ablation and training curves.}
  \label{fig:robust_results1}
\end{figure}

\subsection{Quantitative Results and Ablation Studies}
\begin{figure}[t]
  \centering
  \includegraphics[width=0.9\linewidth]{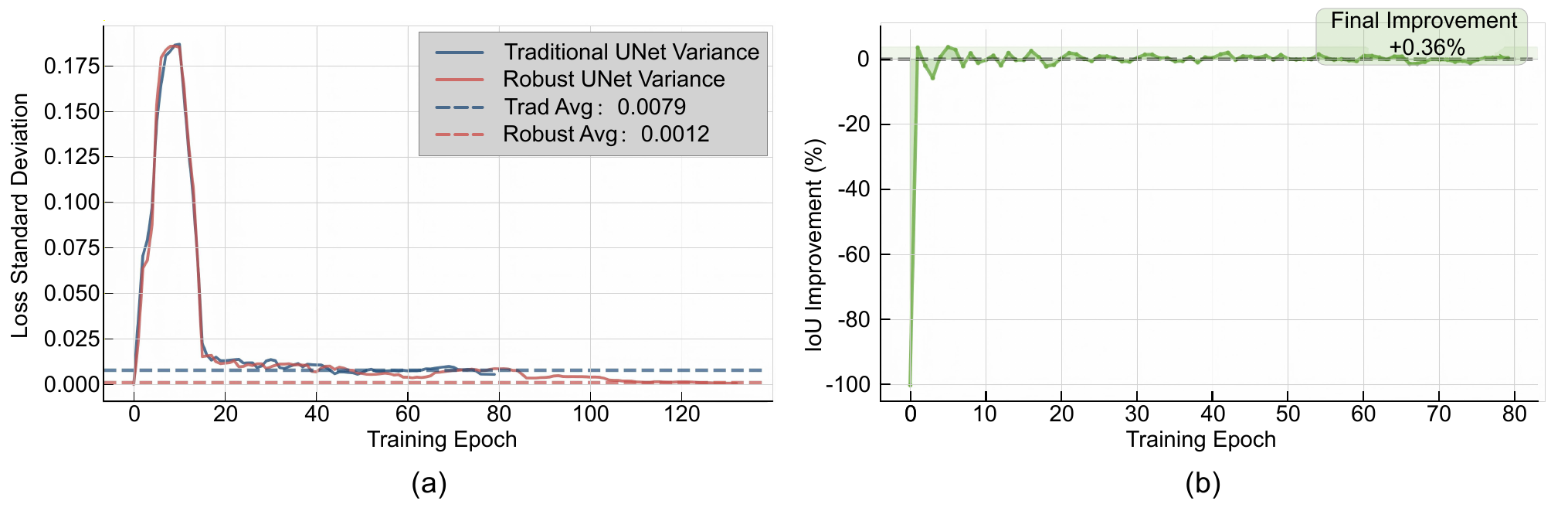}
  \caption{Evaluation of the proposed robust enhancement framework.}
  \label{fig:robust_results2}
\end{figure}

\begin{figure}[t]
  \centering
  \includegraphics[width=0.9\linewidth]{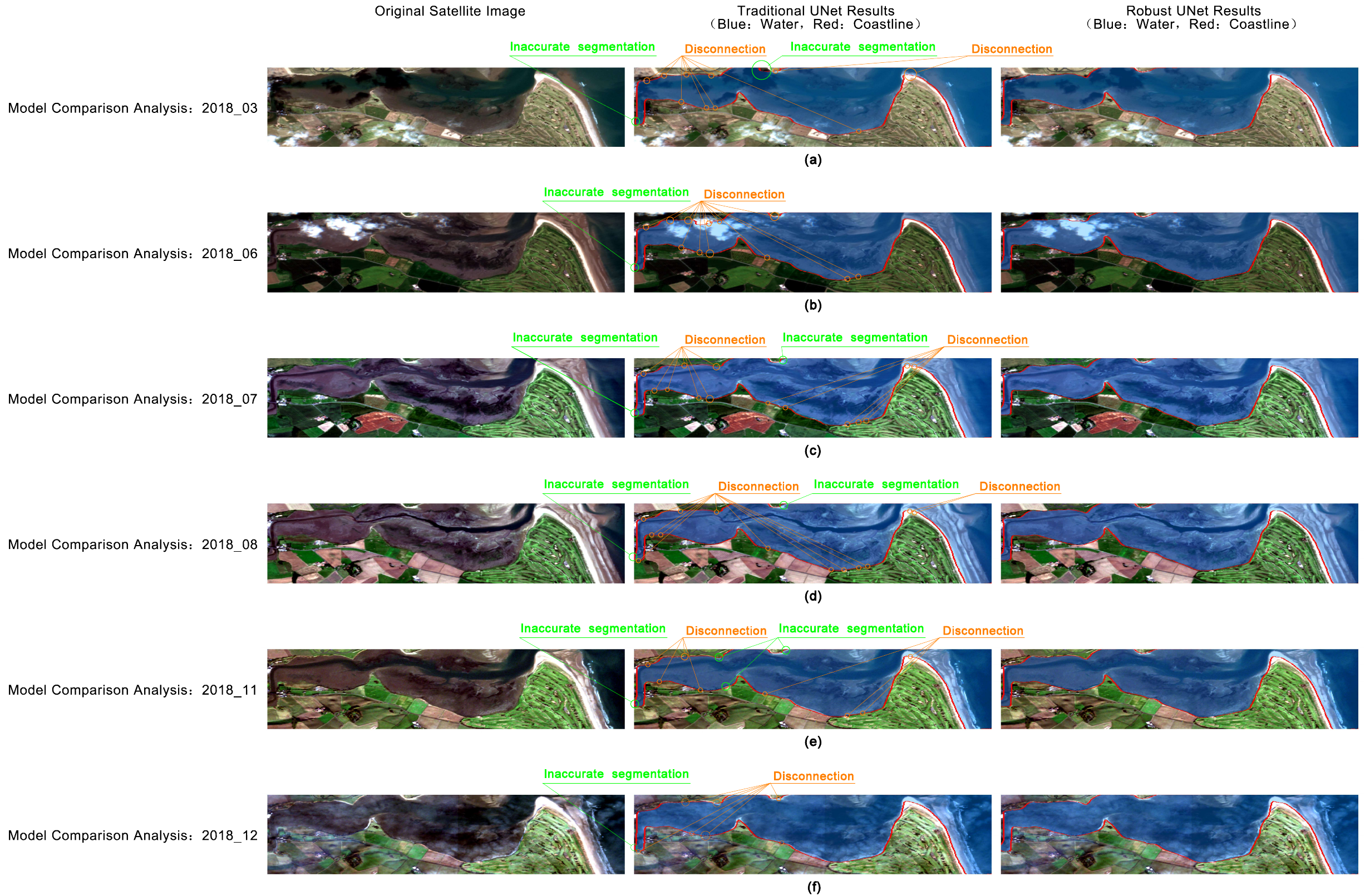}
  \caption{Segmentation examples comparison between the proposed Robust U-Net and vanilla U-Net.}
  \label{fig:model_comp}
\end{figure}

Fig.~\ref{fig:robust_results1} highlights both the individual and overall benefits of the enhancement framework. Sub-figure (a) shows that HSV-guided supervision yields the strongest impact (0.85), followed by gradient-based coastline optimization (0.79) and morphological refinement (0.74), while connectivity control (0.68) and sea cleanup (0.72) provide additional stability. Fig.~\ref{fig:robust_results2} (a) and (b) confirm that the Robust U-Net converges as fast as the baseline while achieving higher accuracy and IoU with significantly reduced variance (from $7.9\times10^{-3}$ to $1.2\times10^{-3}$, an 84\% drop). 
Sub-figure (d) further shows a stable IoU gain of +0.36\% in later epochs, demonstrating that the integrated constraints improve not only accuracy but also training stability and reliability. 

\subsection{Qualitative Robustness Analysis}
\label{sec:qual}
As illustrated in Fig.~\ref{fig:model_comp}, the Robust U-Net consistently outperforms the traditional U\mbox{-}Net across temporal scenes. While the baseline often suffers from inaccurate segmentation of vegetation or farmland as water, broken or disconnected coastlines, and unstable predictions under seasonal or illumination changes, our framework produces smoother and more continuous shoreline boundaries, effectively suppresses false positives, and preserves realistic connectivity. 

\begin{figure}[t]
  \centering
  \includegraphics[width=0.9\linewidth]{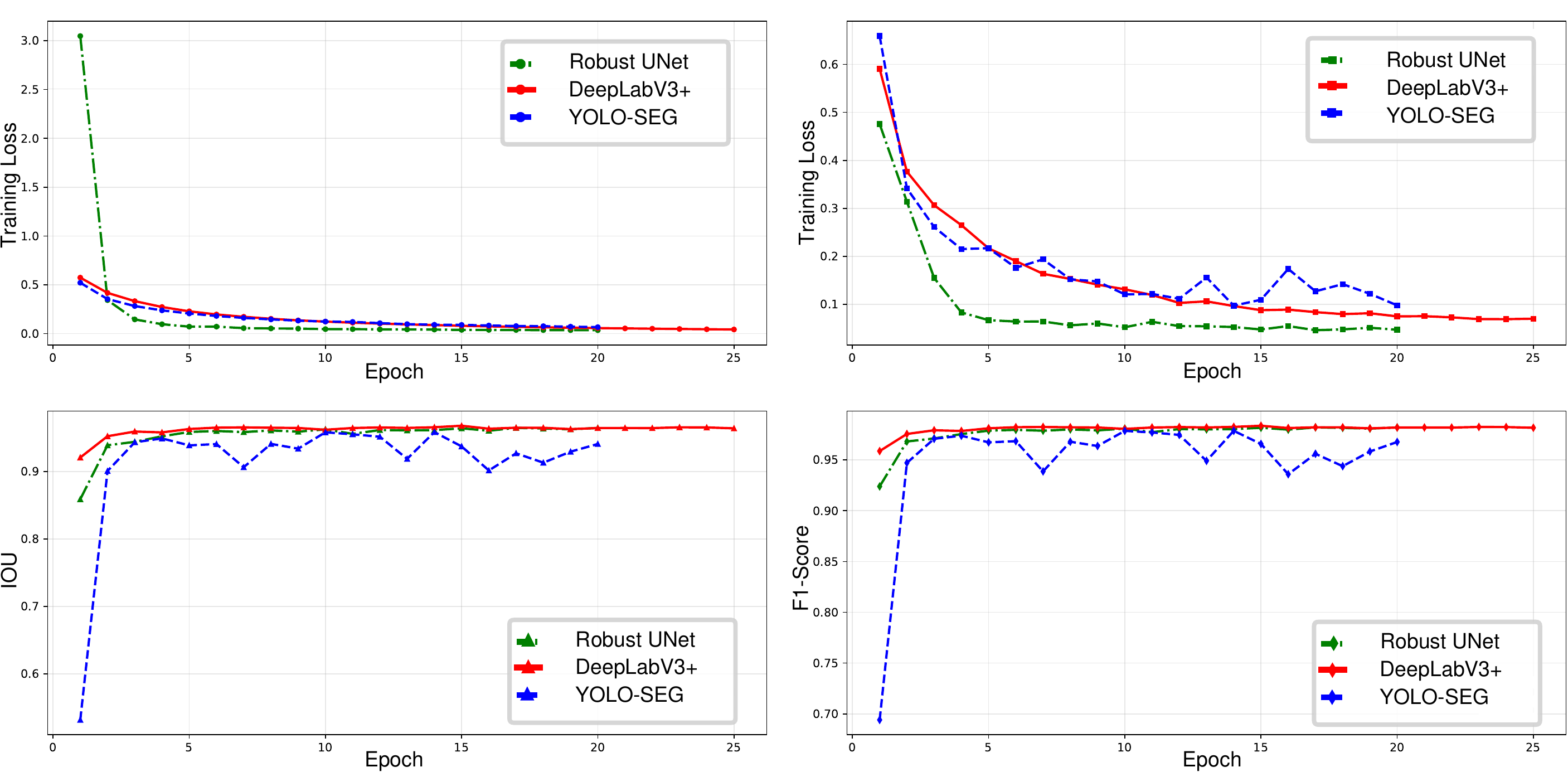}
  \caption{Training results comparison with popular benchmarks.}
  \label{fig:b_comp}
\end{figure}

\subsection{Comparison with other benchmarks}
Robust U\mbox{-}Net is further enhanced with dropout regularization, multi-scale dilated convolutions, channel and spatial attention, and Kaiming initialization for improved stability. As shown in Fig.~\ref{fig:b_comp}, these modifications lead to faster convergence, lower training and validation loss, and consistently superior segmentation accuracy compared to DeepLabV3+ and YOLO-SEG. The Robust U\mbox{-}Net achieves stable IoU around 0.95--0.96 and F1-scores above 0.98, clearly outperforming the less stable DeepLabV3+ and the volatile YOLO-SEG. This confirms that attention mechanisms and multi-scale features substantially improve both convergence and generalization in coastal water segmentation; models' performances are shown in Table~\ref{tab:final_eval}.
\begin{table}[ht]
\centering
\caption{Evaluation Results among Different Models}
\label{tab:final_eval}
\renewcommand{\arraystretch}{1.8}
\setlength{\tabcolsep}{1pt} 
\scriptsize 
\begin{tabular}{l|ccc}
\hline\hline
\textbf{Model} & \textbf{IoU} & \textbf{F1} & \textbf{Acc} \\
\hline
\textbf{Robust U\mbox{-}Net} & $\mathbf{0.9645 \pm 0.003}$ & $\mathbf{0.9819 \pm 0.002}$ & $\mathbf{0.9810 \pm 0.002}$ \\
DeepLabV3+          & $0.9639 \pm 0.005$ & $0.9816 \pm 0.003$ & $0.9806 \pm 0.003$ \\
YOLO-SEG            & $0.9407 \pm 0.076$ & $0.9676 \pm 0.046$ & $0.9684 \pm 0.040$ \\
\hline\hline
\end{tabular}
\end{table}

\section{Conclusions}
We introduce a robust framework for coastal water segmentation that integrates HSV-guided supervision, geometric and morphological constraints, and sea cleanup regularization. Theory and experiments show improved stability, faster convergence, and higher accuracy over baselines. Robust U-Net achieves the best performance (IoU 0.9645, F1 0.9819, accuracy 0.9810) with more reliable boundaries. Future work includes multimodal fusion and dynamic monitoring.

\section*{Acknowledgment}
This work is partially funded by the EU TERRA \# 101189962.

\bibliographystyle{IEEEbib}
\bibliography{strings,refs}

\end{document}